\documentclass[runningheads]{llncs}
\usepackage{booktabs}
\usepackage[T1]{fontenc}
\usepackage{longtable}
\usepackage{hyperref}
\usepackage{nameref}
\usepackage{amsmath}
\usepackage{amssymb}
\usepackage{multirow} 
\usepackage{adjustbox}
\usepackage{algorithm}
\usepackage{framed}
\usepackage{mathtools}
\usepackage{amsthm}
\def\BibTeX{{\rm B\kern-.05em{\sc i\kern-.025em b}\kern-.08em
    T\kern-.1667em\lower.7ex\hbox{E}\kern-.125emX}}
\usepackage{algorithmic}
\newcommand\mydef{\mathrel{\stackrel{\makebox[0pt]{\mbox{\normalfont\tiny def}}}{=}}}

\usepackage[numbers]{natbib}

\theoremstyle{plain}
\theoremstyle{definition}
\theoremstyle{remark}

\newcommand{\rounding}{\circ}

\newcommand{\fp}{\mathbf{F}}

\usepackage{graphicx}
\begin{document}
%
\title{Revisiting 16-bit Neural Network Training: A Practical Approach for Resource-Limited Learning}
\titlerunning{16-bit Neural Network Training}
%
\author{Juyoung Yun\inst{1} \and
Sol Choi\inst{1} \and
Francois Rameau\inst{2} \and
Byungkon Kang\inst{2}\and \\ 
Zhoulai Fu\inst{2,3,*}
}
\authorrunning{J. Yun et al.}
%
\institute{Stony Brook University, Stony Brook, NY, 11794, USA
\email{\{juyoung.yun,sol.choi\}@stonybrook.edu} \and
State University of New York, Korea, Incheon, 21985,  Republic of Korea
\email{\{francois.rameau,byungkon.kang,zhoulai.fu\}@sunykorea.ac.kr} \and
Virginia Tech, Blacksburg, VA, 24061, USA \\
*Corresponding author.
}

\maketitle              
\begin{abstract}
With the increasing complexity of machine learning models, managing computational resources like memory and processing power has become a critical concern. Mixed precision techniques, which leverage different numerical precisions during model training and inference to optimize resource usage, have been widely adopted. However, access to hardware that supports lower precision formats (e.g., FP8 or FP4) remains limited, especially for practitioners with hardware constraints. For many with limited resources, the available options are restricted to using 32-bit, 16-bit, or a combination of the two. While it is commonly believed that 16-bit precision can achieve results comparable to full (32-bit) precision, this study is the first to systematically validate this assumption through both rigorous theoretical analysis and extensive empirical evaluation. Our theoretical formalization of floating-point errors and classification tolerance provides new insights into the conditions under which 16-bit precision can approximate 32-bit results. This study fills a critical gap, proving for the first time that standalone 16-bit precision neural networks match 32-bit and mixed-precision in accuracy while boosting computational speed. Given the widespread availability of 16-bit across GPUs, these findings are especially valuable for machine learning practitioners with limited hardware resources to make informed decisions.

\end{abstract}

\section{Introduction}
\label{sect:intro}

As machine learning models grow in complexity and size, the demand for computational resources such as memory and processing power has become a critical challenge. To address this, the mixed precision technique has been developed to balance resource consumption and computational efficiency~\cite{mixed-prec-training,wang2018training,bin-connect,xi2023training}. This approach strategically employs different numerical precision types during forward and backward propagation, weight storage, or inference. For instance, GPT-3 was trained with mixed precision of 32-bit and 16-bit formats~\cite{brown2020language}, while Meta's LLaMA 3 employs FP8 precision in both inference and forward pass training phases~\cite{huggingface2024,pytorch2024}.

The trend toward mixed precision is closely linked to the development of lower-precision numerical formats. The ongoing ML arithmetic standard IEEE P3109 has formalized seven 8-bit floating-point formats~\cite{IEEE_P3109_2024}, and the Nvidia Blackwell chip is pioneering 4-bit precision~\cite{DBLP:journals/corr/abs-2402-17764,DBLP:journals/corr/abs-2402-05147}. However, while mixed precision is designed to mitigate resource pressures, access to this technology often remains elusive for many machine learning practitioners facing hardware limitations.

Table~\ref{tab:fp_accessibility} shows major Nvidia GPUs and their hardware support for floating-point operations. Both FP32 and FP16 are supported across all GPUs listed, while support for newer formats like bfloat16 and FP8 is limited to recent architectures~\cite{nvidia_tensor_cores, nvidia_ampere_tuning_guide}. For practitioners constrained by hardware, the options for training models are typically limited to using 32-bit, 16-bit, or their combination in mixed precision.

\setlength{\tabcolsep}{8pt}
\begin{table}[H]
\centering
\footnotesize
\caption{Supported floating-point precisions across major NVIDIA GPU architectures, showing both CUDA Core and Tensor Core capabilities.~\cite{nvidia_tensor_cores, nvidia_ampere_tuning_guide}}

\begin{adjustbox}{width=1.0\columnwidth}
\begin{tabular}{@{}llllll@{}}
\toprule
& \multicolumn{5}{c}{\textbf{Nvidia GPU Architectures}} \\
\cmidrule(lr){2-6} 
 & \textbf{Blackwell} & \textbf{Hopper} & \textbf{Ada Lovelace} & \textbf{Ampere} &\textbf{ Turing} \\
\midrule
\textbf{Supported}   & FP16, & FP16, & FP16, & FP16, & FP16, \\
\textbf{CUDA*}       & FP32, FP64, & FP32, FP64, & FP32, FP64, &FP32, FP64 & FP32, FP64 \\
\textbf{Core precision} & bfloat16 & bfloat16 & bfloat16 & & \\
\midrule
\textbf{Supported }  & FP16, FP64, & FP16, FP64, & FP16, FP64, & FP16, FP64, & FP16 \\
\textbf{Tensor}  & bfloat16, TF32, & bfloat16, TF32& bfloat16, TF32& bfloat16, TF32& \\
\textbf{Core precision }& FP8, FP6, FP4& FP8& FP8&& \\
\bottomrule
\end{tabular}
\end{adjustbox}
\label{tab:fp_accessibility}
\end{table}

Among these options, 16-bit precision should theoretically be preferred, as it offers memory and performance advantages while being well-supported by modern hardware. Yet, surprisingly, there's limited information about the ability of 16-bit formats to work independently. This work focuses on IEEE binary16 due to its wider accessibility compared to Google Brain's bfloat16.

The absence of research on standalone 16-bit functionality can be attributed to two factors: the shift towards even lower precision models by those with advanced hardware, and the observation of numerical instability in early experiments that led to mixed-precision solutions~\cite{wang2018training}. While direct attempts to train networks using 16-bit precision often fail in frameworks like TensorFlow or PyTorch, we show these issues stem from incorrect parameter settings, particularly in the Adam optimizer's epsilon value.

Our work proceeds as follows. We start with a theoretical examination of accuracy differences between 16-bit and 32-bit formats, proposing a lemma that ensures comparable classification results under specific conditions. Through extensive comparisons against 32-bit and mixed-precision training, we demonstrate the viability of standalone 16-bit training in image classification using CNN and ViT~\cite{dosovitskiy2021image} architectures. To ensure reproducibility, we provide detailed experimental settings and will release our code upon publication.
\section{Related Work}
\label{sect:relwork}
In the past decade, we have observed that scaling networks significantly improves their abilities to resolve complex tasks \cite{krizhevsky2012imagenet}. These increasingly large networks require extensive computations for their training and large-scale deployment. This trend raises several issues regarding accessibility, reproducibility, efficiency, environmental impact, and real-time execution. Given these critical challenges, numerous approaches have been proposed to address them and offer solutions to reduce energy and memory footprints while decreasing both running and training time. Quantization techniques reduce the precision of weights and activations to improve efficiency, as demonstrated by Jacob et al. \cite{quant-nn}, while methods like PACT \cite{choi2018pact} further optimize performance by learning clipping ranges. Pruning, as explored by Han et al. \cite{han2015learning}, removes redundant parameters to compress models, and Gale et al. \cite{gale2019state} have shown its effectiveness in reducing resource usage without significant accuracy loss. Distillation, introduced by Hinton et al. \cite{hinton2015distilling}, transfers knowledge from large models to smaller ones, and Touvron et al. \cite{touvron2021training} have adapted it for transformers.\\

\noindent\textbf{Low Precision Training Pipeline.} Training neural networks with low-bit precision while maintaining accuracy has been a critical focus in the literature. Early work \cite{low-prec-sgd} showed that using low-precision multiplications can yield results comparable to 32-bit training, thanks to selective precision adjustments. Recent work such as LLM.int8() \cite{dettmers2022llm} introduces efficient 8-bit matrix multiplication techniques for transformers, though requiring careful handling of sensitive layers where quantization errors may arise. At the core of low-precision approaches are two main types: floating-point and fixed-point formats. Floating-point numbers, commonly used in deep learning, can represent a wide range of values crucial for complex tasks requiring high precision. In contrast, fixed-point formats \cite{fixed-point} use a fixed allocation of bits for integer and fractional parts. While offering faster computation and lower memory usage, fixed-point formats face fundamental challenges with range and precision, particularly struggling with tasks that demand wider dynamic range. Studies like \cite{dl-limprec}, \cite{fxpnet}, and \cite{fp-cnn} demonstrate how these limitations lead to accuracy drops in neural networks, especially in larger, more complex models where accuracy is highly sensitive to weight and activation precision.
Integer quantization approaches, including \cite{int-dnn}, \cite{mixed-prec-cnn}, and \cite{quant-nn}, attempt to address these challenges but still face accuracy degradation in complex models. Even recent advances like SmoothQuant \cite{smoothquant} and FlexPoint \cite{flexpoint} require careful tuning. For example, while Q8BERT \cite{zafrir2019q8bert} shows promise with 8-bit quantization, TernaryBERT \cite{zhang2020ternarybert} faces significant accuracy challenges due to its lower bit-width.
Our research proposes a simple yet effective method using IEEE 16-bit floating-point operations for both training and inference, avoiding the inherent limitations of fixed-point representations while maintaining the precision necessary for modern deep learning. \\

\noindent\textbf{Mixed Precision.} The inception of practical low-precision mechanisms traces back to Courbariaux et al. \cite{courbariaux2015binaryconnect}, who devised BinaryConnect. XNOR-Net \cite{rastegari2016xnor} extended this by introducing binary operations into CNNs. Zhou et al. \cite{zhou2016dorefa} introduced DoReFa-Net for edge devices. Micikevicius et al. \cite{mixed-prec-training} proposed mixed precision training using 16-bit operations while maintaining 32-bit weights. Wang et al. \cite{wang2018training} developed an 8-bit training method with 16-bit accumulations. Koster et al. \cite{koster2020bfloat16} revisited BF16, making it effective for deep learning tasks without significant accuracy loss. \\

\noindent\textbf{Hardware Support.} Sharma et al. \cite{bitfusion} introduced BitFusion, a dynamic hardware accelerator optimizing resource use. Fixed-point arithmetic is supported by FPGAs and ASICs \cite{fixed-point}, and fixed-point DSP processors are commonly used in embedded systems \cite{fpga-dsp}. Google's TPU supports BF16 \cite{koster2020bfloat16}, while NVIDIA's recent architectures expanded support for various precisions \cite{Nvidia}. While general-purpose GPUs support low-precision formats, specialized hardware such as BitFusion, TPUs, and fixed-point devices are essential for maximizing the benefits of low-precision computing. \\

Our work addresses this gap by rigorously investigating the efficacy of standalone IEEE 16-bit floating-point precision in various neural network architectures. Through both theoretical and experimental analysis, we provide evidence that standalone 16-bit precision can achieve comparable performance to 32-bit and mixed-precision approaches while offering significant improvements in computational efficiency. This study paves the way for more accessible, energy-efficient machine learning applications, especially for practitioners who may not have access to high-end hardware but still require efficient, low-precision solutions.


\section{Theoretical Analysis}
\label{sect:theory}
This section presents a theoretical comparison between standalone IEEE 16-bit and 32-bit floating-point deep neural networks. 
The objective is to  study the reasons for any observed differences or similarities in their performance.

\subsection{Background: Floating-Point Error}

We denote the sets of real numbers and integers by $\mathbb{R}$ and $\mathbb{Z}$. We write $\fp_{32}$ and $\fp_{16}$ for the sets of IEEE 32-bit and 16-bit floating-point numbers, respectively, excluding $\pm\infty$ and NaN (Not-a-Number). Following the IEEE-754 standard of floating-point formats~\cite{4610935}, an IEEE 16-bit floating-point number can be exactly represented as a 32-bit number by padding with zeros, namely, the following inclusion relationship holds:	$\fp_{16} \subset \fp_{32} \subset \mathbb{R}.$

Rounding is necessary when representing real numbers that cannot be exactly represented. The rounding error of a real number $x$ at precision $p$ is defined as $|x - \rounding_p(x)|$, where $\rounding_p : \mathbb{R} \to \fp_p$ is the rounding operation that maps $x$ to its nearest floating-point number in $\fp_p$, namely $\rounding_p(x) \mydef \mathrm{argmin}_{y \in \fp_p} |x - y|$. (This definition of the rounding operation ignores the case with a tie for simplicity.)

Floating-point error includes rounding error and its propagation.  For example, the floating-point code {\tt sin(0.1)} goes through three approximations. First, 0.1 is rounded to the floating-point number $\rounding_p(0.1)$ for some precision $p$. Then, the rounding error is propagated by the floating-point code {\tt sin} proportional to its condition number~\cite{higham2002accuracy}. Lastly, the output of the calculation is rounded again if an exact representation is not possible.

Floating-point error is usually measured in terms of absolute error or relative error.  This paper uses the absolute error measurement, which quantifies the difference between two floating-point numbers $x$ and $y$ by a straightforward $|x - y|$.

\subsection{Theory:  Error vs.  Tolerance}

\paragraph{Notation.} A dataset $D\subseteq \mathcal{X}\times \{0,1,\ldots, n-1\}$ is a set of pairs of samples and labels. A classifier $M_r$ of precision $r\in \{16,32\}$, also known as classification model, is a function $\mathcal{X}\rightarrow ([0,1]\cap \fp_r)^n$ corresponding to the 16-bit and 32-bit floating-point formats,  
 
\begin{definition}
Let $x\in \mathcal{X}$ be a sample input. Suppose $M_{16}$ and $M_{32}$ are 16-bit and 32-bit deep learning models trained by the same neural network architecture and hyper-parameters. We define the \emph{floating point error} between the classifiers on $x$ as
	\begin{equation}
		\label{eq:delta}
  		\delta (M_{32}, M_{16}, x)\mydef\max\{d|d\in\vert M_{32}(x) - M_{16}(x)\vert\},
	\end{equation}
where $d\in\vert M_{32}(x) - M_{16}(x)\vert$ denote that $d$ is an element of the vector $\vert M_{32}(x) - M_{16}(x)\vert$.
\end{definition}

Below, we investigate the degree to which the floating error affects the model accuracy.
We first formalize the concept of \textit{classification tolerance} and then prove a lemma, which provides a theoretical guarantee on the model accuracy of the 16-bit floating-point neural network with regard to its 32-bit counterpart.  

\begin{definition}
	The \emph{classification tolerance} of a classifier $M$ with respect to an sample $x\in\mathcal{X}$ is defined as  the difference between the largest probability and the second-largest one:
	\begin{equation}
		\Gamma(M,x)\mydef p_0 - p_1,
		\label{eq:gamma}
	\end{equation}
	where $p_0=\max \{p | p \in M(x)\}$, and $p_1=\max \{p | p \in M(x) \text{ and } p \neq p_0\}$. Here,  $p\in M(x)$ denotes that $p$ is an element of the vector $M(x)$.
\end{definition}

\begin{lemma}\label{lem:twice}
Let $\mathrm{class}(M_r, x)$ denote the classification result of classifiers $M_r$ for $r\in\{16,32\}$ on a sample $x\in \mathcal{X}$. Namely, $\mathrm{class}(M_r, x)\mydef\mathrm{argmax}_{i}\{p_i\vert p_i\in M_r(x)\}$. We have: If  
	\begin{align} \label{eq:twice}
		\Gamma( M_{32},x) \geq 2 	\delta (M_{32}, M_{16}, x),
	\end{align}
	then $\mathrm{class}(M_{32}, x)=\mathrm{class}(M_{16}, x)$. 
\end{lemma}

\begin{proof}
	Let $M_{32}(x)$ be $(p_0, .... p_{N-1})$.  Without loss of generality we assume $p_0$ is the largest one in $\{p_i\}$ ($0\leq i \leq N-1$). 
	We denote $\delta(M_{32}, M_{16}, x)$ by $\delta$.  Following Eq. \ref{eq:twice}, 	we have 
	\begin{align} \label{eq:gammabis}
		\forall i\in \{1,..., N-1\}, p_0  - p_i \geq  2  \delta.
	\end{align}
	Let $M_{16}(x)$ be $(p'_0, .... p'_{N-1})$. Then for each  $i\in \{1,..., N-1\}$, we have
	\begin{align*}
		p_0' &\geq p_0 - \delta  && \text{By Eq. \ref{eq:delta} and Def. of }  p_0, p_0'  \\ \nonumber
		&\geq p_i+  \delta && \text{By Eq. \ref{eq:gamma}  and  Eq. \ref{eq:gammabis}}\\ \nonumber
		&\geq p_i' &&\text{By Eq. \ref{eq:delta}}  \nonumber
	\end{align*}
	Thus $p_0'$ remains the largest in the elements of $M_{16}(x)$. 
\end{proof}

\paragraph{Explanation of the lemma.}
 Suppose $M_{32}(x)=(p_0, p_1, \ldots, p_{n-1})$ for a classification problem with  $n$ labels and let $p_0$ and $p_1$  be the largest and second largest probabilities in $\{p_i | i \in [1,n-1]\cup \mathbb{Z}\}$. Suppose $M_{16}(x) = (p_0', p_1', \ldots, p_{n-1}')$.  In the worst case (in terms of the difference between  $M_{16}(x)$ and $M_{32}(x)$), $p_0'$ can drop to $p_0 - \delta (M_{32}, M_{16}, x)$, and $p_1'$ can increase to $p_1+\delta (M_{32}, M_{16}, x)$. But as long as $p_0- \delta (M_{32}, M_{16}, x) > p_1+ \delta (M_{32}, M_{16}, x)$, the two classifiers $M_{16}$ and $M_{32}$ must have the same classification result on $x$.

\subsection{Observation: Standalone IEEE 16-bit Floating-Point DNN on MNIST}

We validate Lemma~\ref{lem:twice} by comparing the performance of a standalone IEEE 16-bit floating-point deep neural network to a 32-bit one on MNIST. Our implementation is done using TensorFlow. The 16-bit implementation follows the same architecture as the 32-bit counterpart, except that all floating-point operations are performed using 16-bit precision.

\begin{figure*}[htbp]
  \centering
  \begin{minipage}{0.48\textwidth}
    \centering
    \includegraphics[width=\textwidth]{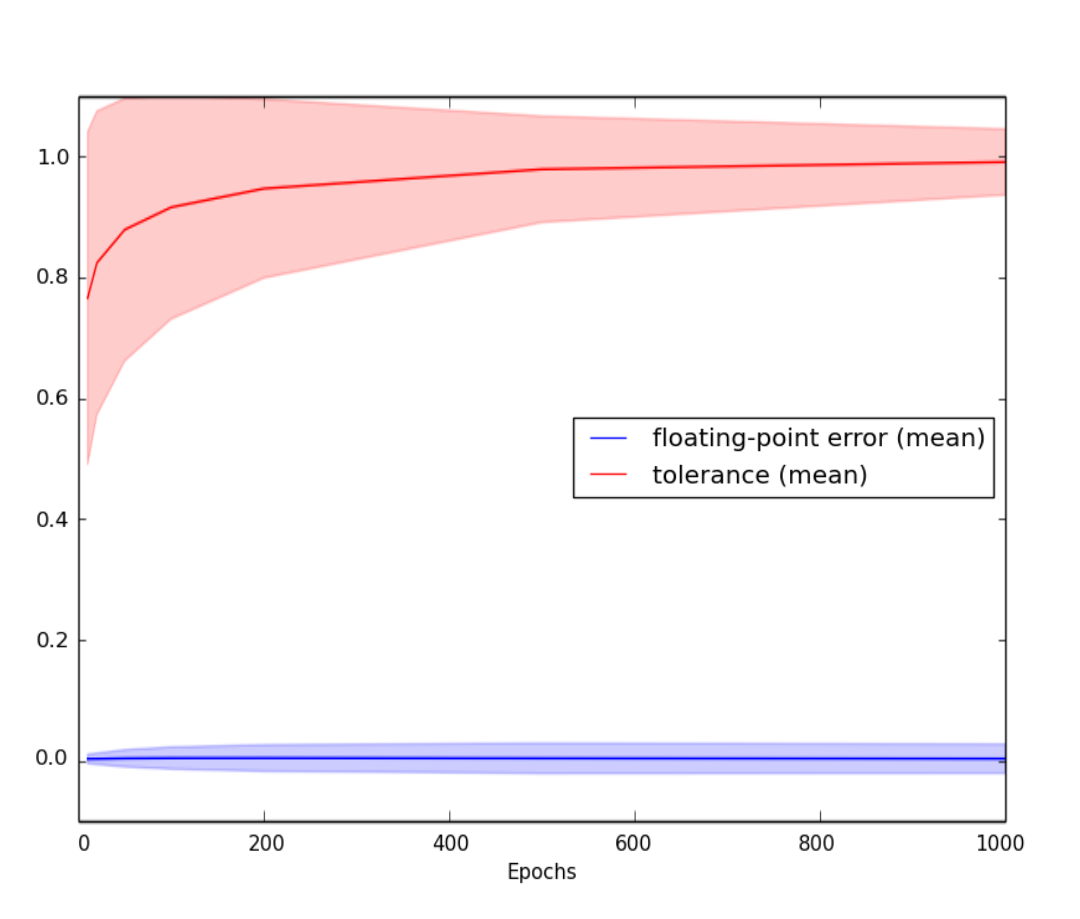}
    \vspace{-0.3cm}
    \caption{Floating-point Errors vs. Classification Tolerance: Mean $\pm$ standard deviation}
    \label{fig:err_tol}
  \end{minipage}
  \hfill
  \begin{minipage}{0.48\textwidth}
    \centering
    \includegraphics[width=\textwidth]{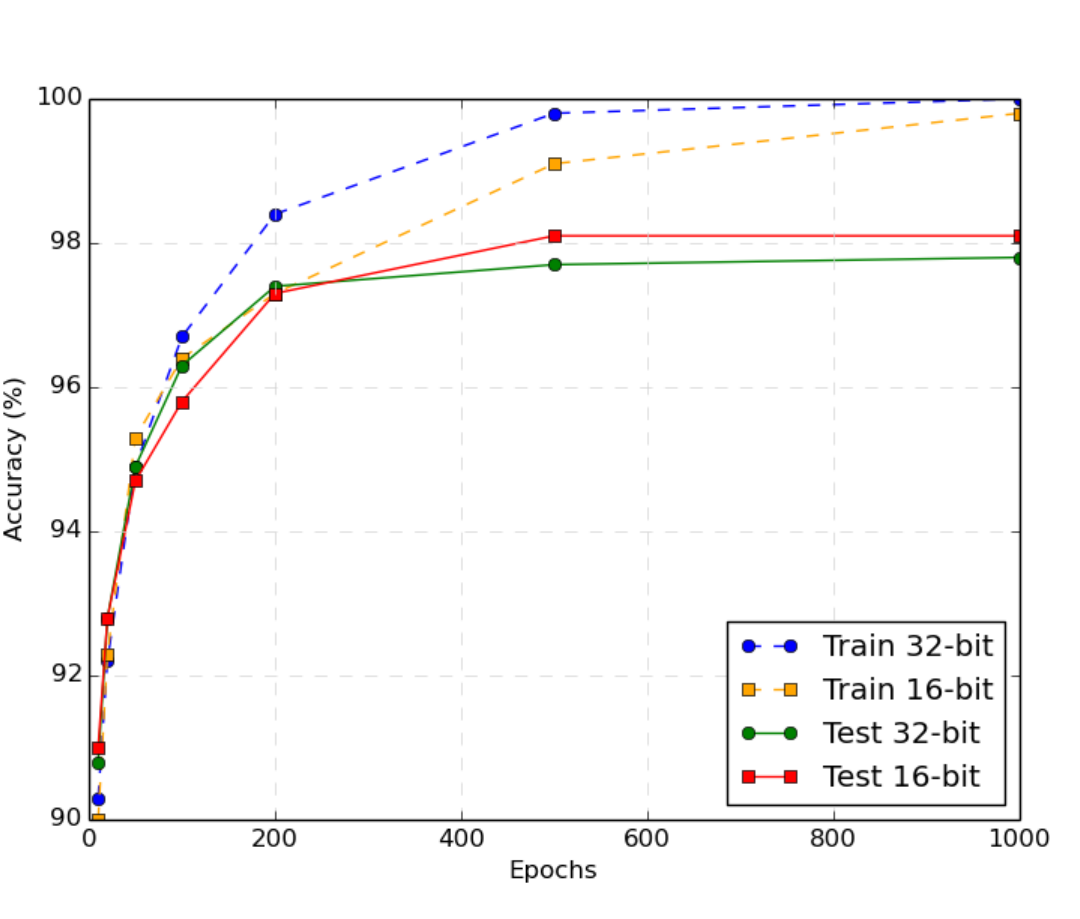}
    \vspace{-0.3cm}
    \caption{DNN Accuracies on MNIST Dataset: 32-bit vs. 16-bit floating-point}
    \label{fig:dnn_accuracy}
  \end{minipage}
\end{figure*}

Fig.~\ref{fig:err_tol} presents the results of our experiments, showing the classification tolerance $\Gamma$ and floating-point error $\delta$. Our findings indicate that the mean floating-point error has a magnitude of approximately 1E-3 with a variance of up to 1E-4, while the classification tolerance is approximately 1E-1 with a variance of 1E-2. In other words, the average classification tolerance is about 100 times larger than the average floating-point error, suggesting that Eq.~\ref{eq:twice}, namely $\Gamma > 2\delta$, holds for most samples in the MNIST dataset. Additionally, we observed that in specific cases, floating-point errors can surpass the tolerance, reaching values as high as 9.95E-01 (presumably due to overflow/underflow). As a result, we anticipate that our 16-bit and 32-bit implementations will yield highly similar yet slightly different accuracy results.

Our experiments confirm this expectation. Fig.~\ref{fig:dnn_accuracy} shows the accuracy results on MNIST comparing the performance of the two implementations. We observe that the training accuracy of the 16-bit model is close to that of the 32-bit model, while the 16-bit model achieves slightly better accuracy in terms of testing accuracy. 

This finding motivates us to investigate whether the similarity between the IEEE 16-bit and 32-bit models holds true for more complex neural networks. In theory, if a certain percentage, denoted by $a\%$, of samples in a dataset satisfy Eq.~\ref{eq:twice}, Lemma~\ref{lem:twice} guarantees that the 32-bit and 16-bit models will have at least $a\%$ of their classification results in common. However, determining the exact  $a\%$ poses a challenge. Nevertheless, we cautiously believe that for complex neural networks, a large majority of data satisfies Eq.~\ref{eq:twice}. This belief stems from the fact that the loss function for classification problems, in the form of cross-entropy $-\Sigma\log(p_i)$, guides the probabilities $p_i$ towards 1 during training, thereby resulting in a larger error tolerance $\Gamma$ compared to the relatively smaller $\delta$. 

\paragraph{Remark.}
Previous research has highlighted the ability of neural networks to tolerate noise and even act as a regularization method. However, our proposed analysis differs in two ways. Firstly, it is the first comparison of its kind that focuses on standalone IEEE 16-bit neural networks. Secondly, our analysis offers a unique quantification approach that distinguishes it from previous explanations related to regularization. Our approach examines floating-point errors and formalizes the concept of tolerance, allowing for its quantitative comparison with floating-point errors.

\section{Experimental Results}
\label{sect:eval}
We have conducted extensive experiments to address the question of whether IEEE 16-bit floating-point can function independently for image classification tasks. This section presents the results of our comparative study. We measured the wall-clock-time it took for the model to train for 100 epochs and test accuracy.

\begin{table*}\centering
\caption{Performance metrics of standalone IEEE 16-bit, 32-bit, and mixed-precision methods~\cite{mixed-prec-training} on CNN and Vision Transform architectures~\cite{dosovitskiy2021image}, using the CIFAR-10 dataset~\cite{cifar} without data augmentation and pre-trained models. All the experiments are repeated for 50 times, and we report the means and standard deviations.}
\vspace{0.15cm}
\begin{adjustbox}{width=1\textwidth}
\begin{tabular}{@{}lrrrcrrrcrrrcrrr@{}}\toprule

Architectures & \multicolumn{3}{c}{ Accuracy (\%)} && \multicolumn{3}{c}{Time (s)} && \multicolumn{2}{c}{Accuracy Diff.} && \multicolumn{2}{c}{Time Speedup}\\
\cmidrule{2-4} \cmidrule{6-8} \cmidrule{10-11} \cmidrule{13-14}
(\# par. million) & FP16 &FP32& MP&& FP16 & FP32 &MP && FP16 - FP32 & FP16 - MP && FP32 / FP16 & MP / FP16\\
\midrule 
AlexNet (2.09)~\cite{alexnet} & 76.0$\pm{0.3}$ & 75.8$\pm{0.3}$ & 75.9$\pm{0.3}$ && 
                96 & 174 & 150 && 
                0.2 & 0.1 && 
                1.8x & 1.5x \\
                
VGG-16 (33.76)~\cite{vgg} & 83.7$\pm{0.3}$ & 83.9$\pm{0.2}$ & 83.8$\pm{0.2}$ && 
                377 & 857 & 455 && 
                -0.2 & -0.1 && 
                2.2x & 1.2x\\

VGG-19 (38.36)~\cite{vgg} & 83.8$\pm{0.2}$ & 83.9$\pm{0.3}$ & 83.9$\pm{0.2}$ && 
                416 & 937 & 492 && 
                -0.1 & -0.1 && 
                2.2x & 1.2x\\
                    
ResNet-32 (0.47)~\cite{resnet} & 80.9$\pm{0.3}$ & 80.9$\pm{0.4}$ & 80.9$\pm{0.3}$ && 
                    413 & 551 & 483 && 
                    0.0 & 0.0 && 
                    1.3x & 1.2x \\

ResNet-56 (0.86)~\cite{resnet} & 81.6$\pm{0.4}$ & 81.4$\pm{0.6}$ & 81.5$\pm{0.6}$ && 
                    677 & 905 & 795 && 
                    0.2 & 0.1 && 
                    1.3x & 1.2x\\

ResNet-110 (1.78)~\cite{resnet} & 81.8$\pm{0.4}$ & 81.4$\pm{0.5}$ & 81.8$\pm{0.5}$ && 
                    1256 & 1712 & 1486 && 
                    0.4 & 0.0 && 
                    1.3x & 1.2x \\

DenseNet121 (7.04)~\cite{densenet} & 72.1$\pm{0.3}$ & 72.6$\pm{0.4}$ & 73.0$\pm{0.3}$ && 
                    539 & 720 & 641 && 
                    -0.5 & -0.9 && 
                    1.3x & 1.2x \\
                    
DenseNet169 (12.65)~\cite{densenet} & 71.7$\pm{0.5}$ & 72.3$\pm{0.3}$ & 72.1$\pm{0.3}$ && 
                    724 & 966 & 812 && 
                    -0.6 & -0.4 && 
                    1.3x & 1.1x \\

Xception (22.96)~\cite{xception} & 75.9$\pm{0.4}$ & 76.3$\pm{0.3}$ & 76.3$\pm{0.4}$ && 
                    324 & 611 & 412 && 
                    -0.4 & -0.4 && 
                    1.9x & 1.3x \\
MobileNetV2 (2.27)~\cite{mobilenetv2} & 69.4$\pm{1.1}$ & 70.0$\pm{1.2}$ & 70.0$\pm{1.1}$ && 
                    353 & 588 & 411 &&
                    -0.6 & -0.6 &&
                    1.7x & 1.2x \\

VIT-8 (2.02)~\cite{dosovitskiy2021image} & 71.0$\pm{0.3}$ & 71.3$\pm{0.3}$ & 71.2$\pm{0.3}$ && 
                316 & 423 & 410 && 
                -0.3 & -0.2 &&
                1.3x & 1.2x \\

VIT-12 (2.55)~\cite{dosovitskiy2021image} & 71.1$\pm{0.3}$ & 71.4$\pm{0.3}$ & 71.5$\pm{0.3}$ && 
                425 & 663 & 629 && 
                -0.3 & -0.4 &&
                1.6x & 1.5x \\
\midrule
Mean 
&&&&&&&&& -0.2 & -0.2 && 1.6x & 1.3x
\\
\bottomrule
\end{tabular}
\end{adjustbox}
\label{tab:summary}
\end{table*}

\begin{figure*}[t]
  \centering
  \includegraphics[width=0.9\textwidth]{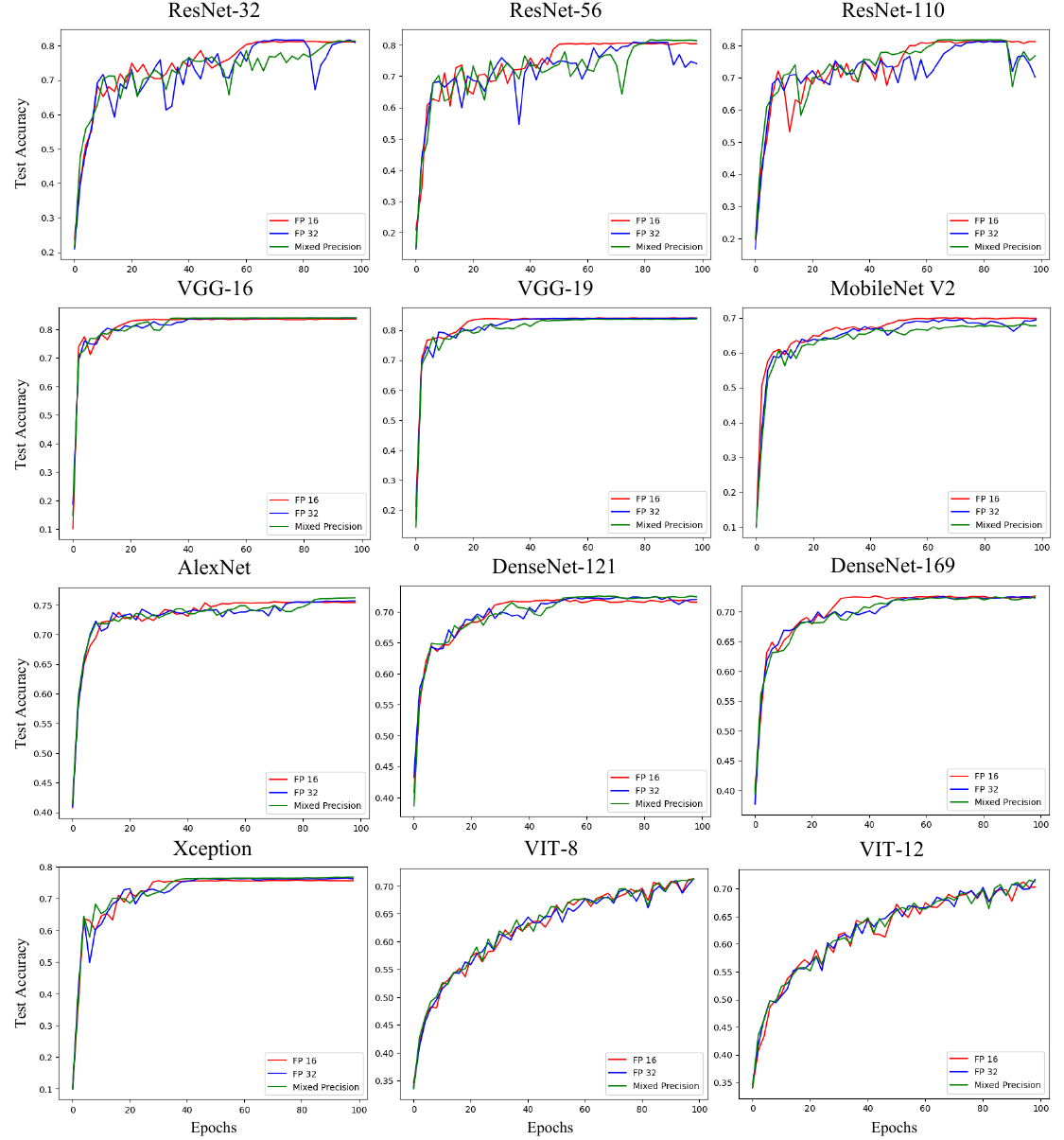}
  \caption{ Comparative test accuracy over 100 epochs on  CNNs and Vision Transformer (ViT) architectures utilizing IEEE 16-bit, 32-bit, and mixed precision.}
  \label{fig:testacc}
\end{figure*}

\subsection{Precisions} \label{sect:baseline}  
We compare three floating-point precision schemes: FP16 (IEEE 16-bit), FP32 (standard 32-bit), and MP (mixed precision). For FP16, we use TensorFlow’s \texttt{tf.keras.backend.set\_floatx('float16')} and ensure no unintended type casting to FP32 occurs during training. FP32 uses TensorFlow’s default 32-bit precision. For MP, we follow TensorFlow’s mixed precision policy via \texttt{tf.keras.} \texttt{mixed\_precision.set\_global\_policy('mixed\_float16')}, which allows FP16 operations with FP32 accumulations~\cite{mixed-prec-training, tensorflow2023, nvidia2023}, as originally proposed by Baidu and NVIDIA.

\subsection{Experimental Settings}
We conduct image classification experiments using the CIFAR-10 dataset~\cite{cifar}, converting it from 32-bit to IEEE 16-bit for FP16 training. Our evaluation spans 12 architectures, including 10 popular CNNs (AlexNet~\cite{alexnet}, VGG16/19~\cite{vgg}, ResNet-34/56~\cite{resnet}, MobileNetV2~\cite{mobilenetv2}, DenseNet-121/169~\cite{densenet}, and Xception~\cite{xception}) and two Vision Transformers (ViT-8 and ViT-12~\cite{dosovitskiy2021image}). All models are trained from scratch without pre-trained weights. Each experiment runs for 100 epochs with a batch size of 256 and a learning rate of 0.01 using SGD~\cite{sgd} with momentum 0.9~\cite{momentum-original}. No data augmentation or hyperparameter tuning is applied. Results are averaged over 50 random seeds, and all experiments are performed on a single RTX 4080 GPU with 40GB RAM, reporting both accuracy and wall-clock training time.



\subsection{Results}

\begin{figure*}[t]
  \centering
  \includegraphics[width=1.0\textwidth]{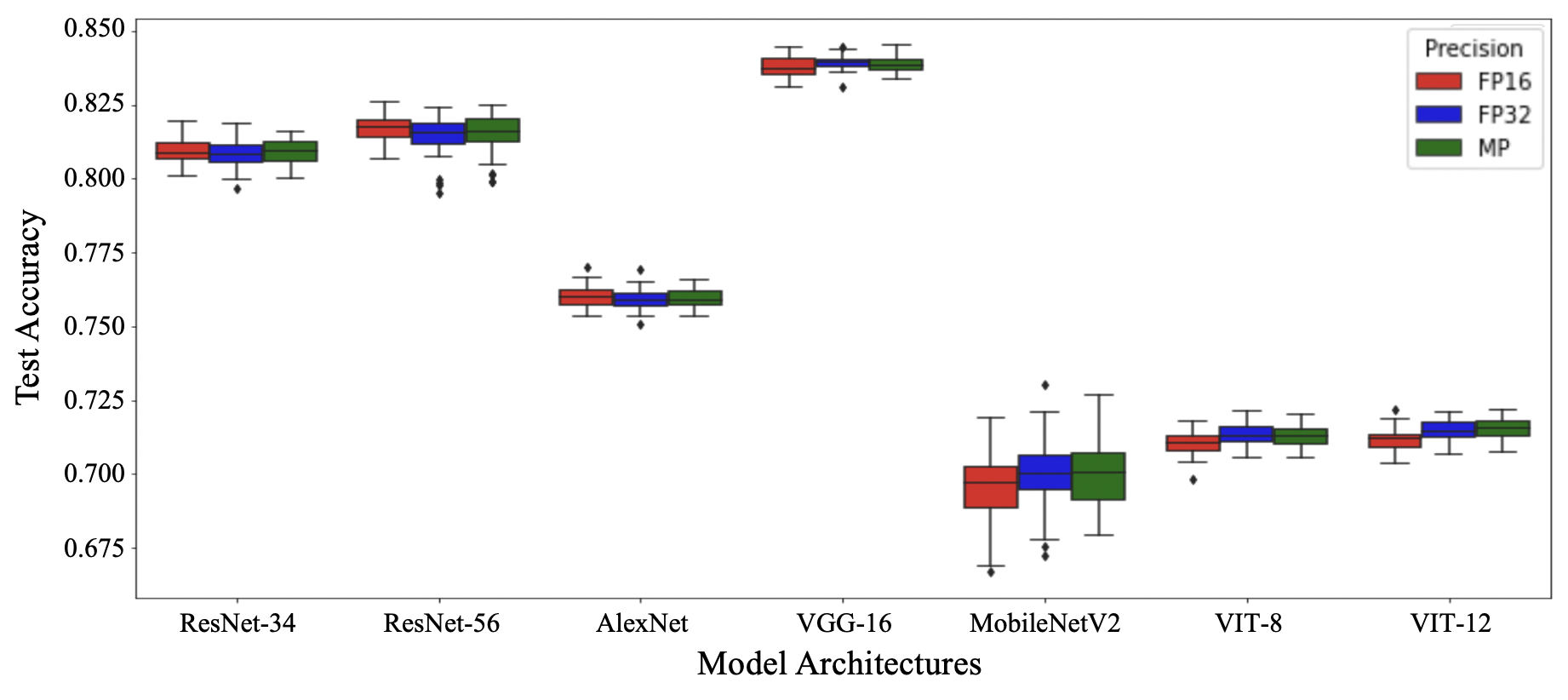}
  \caption{Boxplot of Test Accuracy: This figure illustrates the performance of CNN models and the Vision Transformer across three floating-point precisions: IEEE 16-bit, 32-bit, and mixed precision. Results from 50 random seeds are included to ensure unbiased representation. The  white lines indicate the medians, while the white dots represent outliers.}
  \label{fig:box}
\end{figure*}

As summarized in Tab.~\ref{tab:summary}, our results demonstrate that IEEE 16-bit floating point (FP16) offers a significant computational advantage over both 32-bit (FP32) and mixed precision (MP), with an average speedup of 1.6$\times$ over FP32 and 1.3$\times$ over MP. For instance, FP16 reduces VGG-16 training time from 857s (FP32) to 377s, and Xception from 611s to 324s. In terms of accuracy, FP16 performs competitively, with only a 0.2\% average decrease compared to FP32 and MP across 12 architectures. In some cases, such as AlexNet and ResNet-32, FP16 matches or even slightly outperforms the alternatives. Fig.~\ref{fig:testacc} illustrates test accuracy trends over training epochs, showing no consistent winner across precisions. The boxplot in Fig.~\ref{fig:box}, averaged over 50 random seeds, further shows that FP16 yields similar or better median accuracy with lower standard deviation in most models, though occasional outliers suggest reduced robustness, as discussed in Sect.~\ref{sect:threats}. Overall, FP16 proves to be an efficient and accurate alternative for deep learning under hardware constraints.

\section{Discussion: Numerical Stability and Hyperparameter Tuning}
\label{sect:threats}
While FP16 performs competitively across most architectures, we observe occasional performance drops in a few models due to numerical instability, particularly from overflow or underflow errors inherent to the limited dynamic range of 16-bit floats. As shown in Fig.~\ref{fig:box}, FP16 underperforms in 3 out of 7 architectures when considering outlier cases. However, these anomalies are rare and can be mitigated by averaging results across multiple runs, making FP16 a viable option even under strict memory constraints.

\setlength{\tabcolsep}{8pt} 
\begin{table}[htbp]
\centering
\footnotesize
\caption{Performance analysis of the Adam optimizer was conducted using float16 precision format on CIFAR-10~\cite{cifar}. It was observed that the float16 format did not exhibit optimal performance with the Adam optimizer~\cite{adam} when the epsilon parameter was set below \(1 \times 10^{-5}\). Consequently, the Stochastic Gradient Descent (SGD) optimizer was employed as an alternative, which did not demonstrate any performance issues under these conditions.}
\begin{tabular}{@{}lccccccc@{}}
\toprule
& \multicolumn{7}{c}{FP16 ADAM} \\
\cmidrule(lr){2-8} 
Epsilon & 0.01 & 0.001 & 1e-4 & 1e-5 & 1e-6 & 1e-7 & Time \\
\midrule
ResNet-18~\cite{resnet} & 74.4 & 78.2 & 74.7 & 22.5 & 12.8 & 11.7 & 10.98s \\
ResNet-34~\cite{resnet} & 78.5 & 79.7 & 77.7 & 18.2 & 17.0 & 10.0 & 23.61s \\
\bottomrule
\end{tabular}

\label{tab:bfloat}
\end{table}

We further analyze the impact of optimizer hyperparameters in FP16 training. Table~\ref{tab:bfloat} presents a performance comparison using the Adam optimizer~\cite{adam} on CIFAR-10~\cite{cifar} with varying epsilon values. The results show that FP16 suffers from significant accuracy degradation when $\epsilon < 10^{-5}$, likely due to instability in the denominator of the update rule:
\begin{align}
\text{Adam:} \quad w_t \leftarrow  w_{t-1} - \eta \frac{\hat{m_t}}{\sqrt{\hat{v_t}} + \epsilon}.
\end{align}
A similar concern applies to RMSProp~\cite{rmsprop}. The default $\epsilon=10^{-7}$ in TensorFlow, while suitable for FP32, may cause numerical issues in FP16 when $\hat{v_t}$ or $v_t$ are close to zero. To address this, we adopted SGD~\cite{sgd}, which exhibited stable performance without requiring tuning. This suggests that FP16 networks can avoid instability by using simpler optimizers and avoiding extreme hyperparameter values.


\section{Conclusion}
\label{sect:conc}
This study presents the first comprehensive theoretical and empirical analysis of standalone IEEE 16-bit floating-point neural networks. Contrary to the common reliance on mixed or higher precision, we rigorously demonstrate that 16-bit precision alone can match 32-bit performance in both accuracy and stability across a range of models, while offering substantial improvements in training speed and memory efficiency. These findings are particularly relevant for practitioners with limited hardware access, where support for formats like FP8 or FP4 is unavailable. Our theoretical framework based on floating-point error and classification tolerance provides insights into when 16-bit can safely replace 32-bit precision. While some numerical instability remains in edge cases, our results support broader adoption of 16-bit training as a practical, efficient, and accessible solution for deep learning in resource-constrained environments.

\section{Statements and Declarations} 
The researcher claims no conflicts of interest. This research was partially supported by the National Research Foundation of Korea (NRF) under Grant RS-2022-00165647. This work was also partly supported by the Institute of Information \& Communications Technology Planning \& Evaluation(IITP)-Innovative Human Resource Development for Local Intellectualization program grant funded by the Korea government(MSIT)(IITP-2025-RS-2023-00259678, 50\%). Byungkon Kang also acknowledges support from Moida Inc.

%
%
%
%
%
\bibliographystyle{splncs04}
%
\bibliography{main}

\end{document}